\def\year{2021}\relax
\documentclass[letterpaper]{article} 
\usepackage{aaai21}  
\usepackage{times}  
\usepackage{helvet} 
\usepackage{courier}  
\usepackage[hyphens]{url}  
\usepackage{graphicx} 
\usepackage{natbib}  
\usepackage{caption} 
\usepackage{amsmath}

\DeclareMathOperator*{\argmin}{arg\,min}
\usepackage{times}
\usepackage{graphicx}
\usepackage{latexsym}
\usepackage{amsfonts}
\usepackage{multirow}
\usepackage{adjustbox}
\usepackage[linesnumbered, ruled, vlined]{algorithm2e}
\usepackage{amsmath}
\usepackage{systeme}
\usepackage{booktabs, makecell, multirow, tabularx}
\usepackage{amsthm}

\usepackage{amsmath}
\usepackage{amssymb}
\usepackage{acronym}
\usepackage{xspace}

\acrodef{SIPP}{Safe interval path planning}

\acrodef{nTO-AA-SIPP}{Na\"ive Time-Optimal Any-Angle SIPP}
\acrodef{iTO-AA-SIPP}{Time-Optimal Any-Angle SIPP With Inverted Expansions}

\newcommand{\ntoaasipp}{\ac{nTO-AA-SIPP}\xspace}
\newcommand{\itoaasipp}{\ac{iTO-AA-SIPP}\xspace}

\newtheorem{lemma}{Lemma}
\newtheorem{theorem}{Theorem}
\newtheorem{statement}{Statement}

\begin{document}
\title{Towards Time-Optimal Any-Angle Path Planning With Dynamic Obstacles
}

\author {
    Konstantin Yakovlev,\textsuperscript{\rm 1, 2}
    Anton Andreychuk \textsuperscript{\rm 3} \\
}
\affiliations {
    \textsuperscript{\rm 1} Federal Research Center for Computer Science and Control of Russian Academy of Sciences \\
    \textsuperscript{\rm 2} HSE University \\
    \textsuperscript{\rm 3} Peoples' Friendship University of Russia (RUDN University) \\
    yakovlev@isa.ru, andreychuk@mail.com
}




\maketitle

\begin{abstract}
Path finding is a well-studied problem in AI, which is often framed as graph search. Any-angle path finding is a technique that augments the initial graph with additional edges to build shorter paths to the goal. Indeed, optimal algorithms for any-angle path finding in static environments exist. However, when dynamic obstacles are present and time is the objective to be minimized, these algorithms can no longer guarantee optimality. In this work, we elaborate on why this is the case and what techniques can be used to solve the problem optimally. We present two algorithms, grounded in the same idea, that can obtain provably optimal solutions to the considered problem. One of them is a naive algorithm and the other one is much more involved. We conduct a thorough empirical evaluation showing that, in certain setups, the latter algorithm might be as fast as the previously-known greedy non-optimal solver while providing solutions of better quality. In some (rare) cases, the difference in cost is up to 76\%, while on average it is lower than one percent (the same cost difference is typically observed between optimal and greedy any-angle solvers in static environments).

\end{abstract}

\section{Introduction}
Planning a path for an agent operating in 2D environment is a well-studied problem that is often solved by, first, introducing a graph that represents the environment and the ways the agent can move (e.g. 4-connected grid) and, second, by finding a path on this graph. When the environment is static, A* \cite{hart1968formal} or one of its numerous modifications can be used to find an optimal solution to the problem. Yet even an optimal solution is often only an approximation of the true shortest path in 2D due to the discretization of the workspace imposed by a graph. To mitigate this issue the idea of any-angle path finding has been proposed. Any-angle planners, e.g. Theta* \cite{nash2007}, ANYA \cite{harabor2016}, allow moving from one graph vertex to the other even if there is no correspondent edge in the given graph but the move is valid, i.e. the straight line segment connecting the vertices does not intersect any obstacle.

In this work, we are interested in any-angle path planning when dynamic obstacles are present in the environment and the objective to be minimized is the time to reach the goal. In this problem setting, we have to take the time dimension into account, which poses the first challenge. The second challenge is that there is no guarantee anymore that reaching the goal or any other vertex by a geometrically shortest path will lead to or contribute to finding an optimal solution.

\begin{figure}[t]
    \centering
    \includegraphics[scale=0.5]{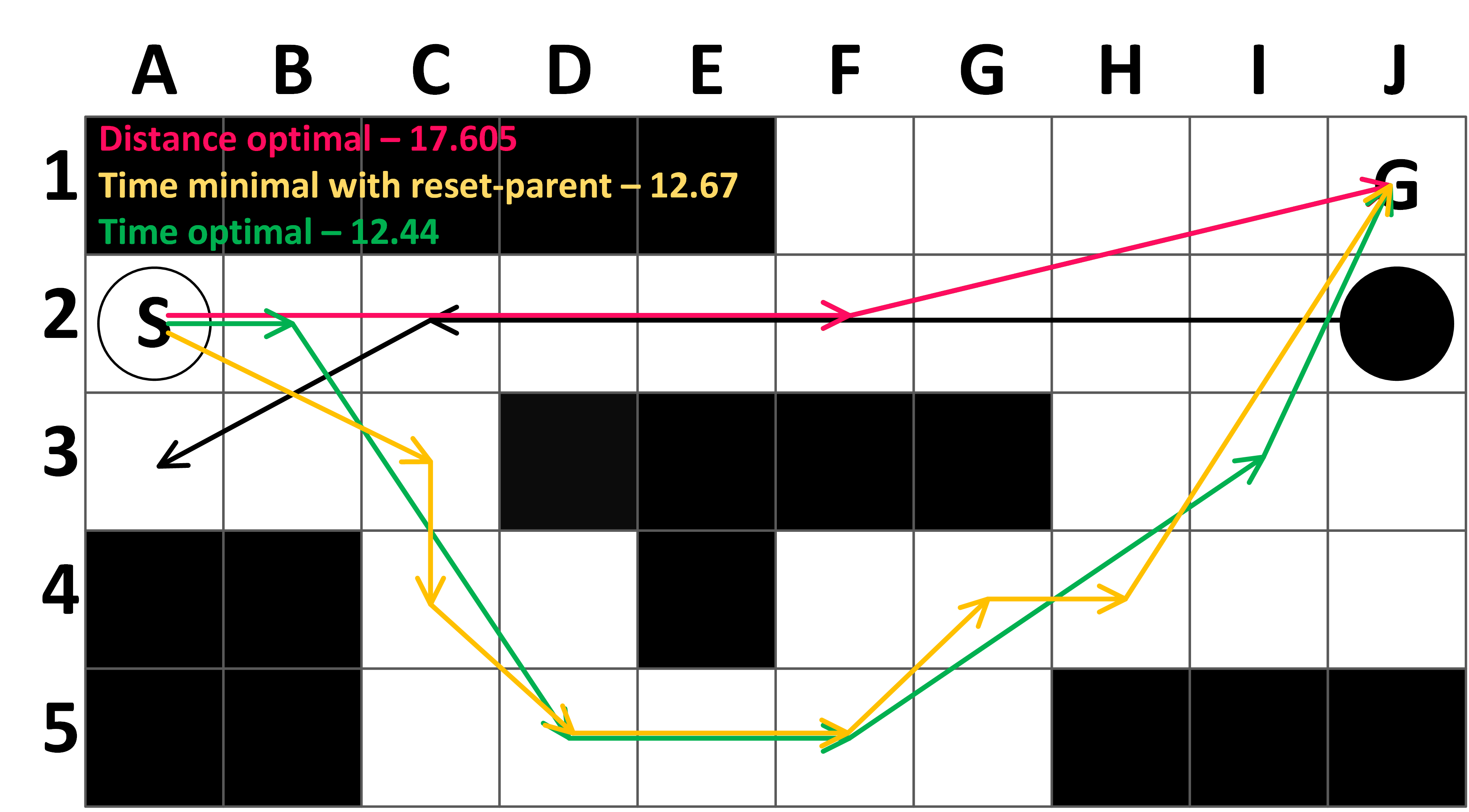}
    \caption{Different any-angle paths for an agent (white disk) on a grid with a dynamic obstacle (black disk). The shortest any-angle path is depicted in red, the time-optimal path -- in green. The radius of the agent/obstacle is 0.4, their moving speed is 1.0. It is beneficial to take a spatial detour to reach the goal as early as possible.
    }
    \label{fig:different-paths}
\end{figure}

Consider an example shown in \figurename~\ref{fig:different-paths}. An agent has to navigate from $A2$ to $J1$ on a 
grid, avoiding a dynamic obstacle, that moves from $J2$ to $A3$ via $C2$. The speed of the agent and the obstacle is the same (1 cell per 1 time unit). The shortest any-angle path from $A2$ to $J1$ is shown in red
. If the agent starts moving immediately, this path can not be safely followed as a collision happens in the vicinity of $E2$, so the agent has to wait. Moreover, the time of that wait plus the time needed to traverse the path is greater than the time needed to execute the detour path shown in green, which is the time-optimal solution in this setting.

Indeed, there exist planners that are capable of handling both dynamic obstacles and any-angle moves, e.g. \cite{yakovlev2017aasipp}, but to the best of our knowledge one that guarantees to find time-optimal solutions is absent. This work aims at filling this gap.

\subsection{Problem Statement}
Consider an agent that navigates the environment discretized to a graph $G=(V, E)$. Vertices correspond to distinct locations and edges -- to the transitions between them. The action set of an agent is composed of actions of two types: \textit{wait} at the current vertex or \textit{move} from one vertex to the other. When moving, the agent follows a straight-line segment, connecting source, and target vertices, with constant speed and with inertial effects neglected.
The cost of an action is its duration. For a move action, its duration equals the length of the corresponding segment divided by the speed of an agent. The duration of a wait action can be any positive number, i.e. the agent can wait in a vertex for any amount of time. 

Two types of moves are distinguished: \textit{i}) regular -- when the agent moves from vertex $u$ to vertex $v$ and the corresponding edge $(u, v)$ is in $G$; and \textit{ii}) any-angle -- when the agent moves from one vertex to the other even though there is no correspondent edge in the graph. To identify whether an any-angle move is valid a dedicated function is given: $los: V \times V \rightarrow \{true, false\}$. It returns $true$ in case a line segment connecting $u$ and $v$ does not intersect any obstacle in the environment w.r.t. agent's size and shape.

A plan is an ordered sequence of time-action pairs: $\pi = (a_1, t_1), (a_2, t_2), ...., (a_n, t_n)$, where $a_i$ stands for actions, and $t_i$ -- for the time moments the agent starts executing them. For a plan to be valid each action should start exactly when the previous one ends. The cost of the plan is the sum of the durations of the constituent actions.

Besides the agent, a fixed number of dynamic obstacles populate the environment and navigate it in the same way the agent does. Their plans, $\pi^1, \pi^2, ..., \pi^k$, are known. We assume that the dynamic obstacles (as well as the agent) do not disappear after accomplishing their plans but rather stay in their target vertices forever. 

Plans for an agent and an obstacle are said to be conflict free if no collision happens between them when they follow these plans. We assume that a collision detection mechanism is given. A plan for an agent, $\pi$, is feasible if it does not conflict with any of $\pi^1, \pi^2, ..., \pi^k$. The problem is to find the least-cost, i.e. \textit{time-optimal}, feasible plan from a given start to a given goal.

\section{Na\"ive Time-Optimal Any-Angle Safe-Interval Path Planning (nTO-AA-SIPP)}

\paragraph{Overview of SIPP} Safe Interval Path Planning \cite{phillips2011sipp} is an algorithm for finding time-optimal paths in environments with dynamic obstacles, whose trajectories are known. It is essentially an A* algorithm that searches through a state-space induced by configuration-interval pairs. The configuration in the considered domain is a graph vertex and the interval is the maximal period of time the vertex can be safely occupied by the agent. Consider the setting depicted in \figurename~\ref{fig:different-paths} and assume that the obstacle starts executing its plan at $t=0$. In that case, it passes through $I2$ and prohibits an agent to occupy the graph vertex which is in the center of that cell from time moment $0.2$ to time moment $1.8$ (not from $0$ to $2$ because of the size of the agent and the obstacle, which is $0.4$ in this example). Thus, SIPP acknowledges two safe intervals and two search nodes are introduced: $n_1=(I2, [0; 0.2])$ and $n_2=(I2, [1.8;+ \infty))$.

The notion of safe intervals is beneficial for several reasons. First, this makes the search space more compact. Second, SIPP naturally allows planning in continuous time. Third, utilizing safe-intervals allows to enumerate \textit{all} states of the search-space beforehand. We will rely on the latter property later on.

\paragraph{Search} SIPP explores the search space in a conventional A* fashion: in each step, the most promising node, $n$, from the frontier is chosen and its successors are generated, thus the node becomes expanded. The most promising node is the one with the minimal $f$-value, which is the sum of $g(n)$ and $h(n)$. $g(n)$ is the time the agent safely reaches $n$ and $h(n)$ is a consistent estimate of time to reach the goal from $n$. 

\paragraph{Generating successors} To generate the successors of $n=(v, [t_1, t_2])$ SIPP iterates through the vertices of the graph that are adjacent to $v$. In the simplest case when any-angle moves are not allowed, successors are generated by considering the transitions defined by the graph edges $(v, v')$. If any-angle moves are allowed and one does not aim at providing optimality guarantees, such moves can be handled by the greedy strategy introduced in~\cite{nash2007}. First, immediate successors, i.e. the ones corresponding to the outgoing graph edges, are generated. Let $n'=(v', [t'_1, t'_2])$ be one of them. Then one checks if the transition to $v'$ is possible from $p(n).vertex$, where $p(n)$ is the predecessor of $n$.
If it is and the time moment the agent arrives to $v'$ following this move belongs to the safe interval of $n'$ and it is smaller compared to arriving to $v'$ using $(v, v')$ edge, 
then the predecessor of $n'$ is reset to $p(n)$ and $g(n)$ is altered appropriately.

This strategy is known to lead to non-optimal solutions even in static environments. Obviously, when dynamic obstacles are present this also holds (see \figurename~\ref{fig:different-paths}, where the plan found by the reset-parent approach is shown in yellow). The ANYA algorithm~\cite{harabor2016} mitigates the issue (for static environments) by reasoning over the sets of geometrically close vertices that are combined together and represent a single search node, characterized by a single $f$-value. Straightforward application of this technique is not possible in the considered domain as being geometrically close for the vertices does not mean that their $g$-values -- earliest arrival times -- are close
. 

\paragraph{nTO-AA-SIPP} To avoid missing \emph{any} valid any-angle move the following approach to successors generation is proposed. When expanding a node $n$, \textit{all} states of the search space, for which a valid transition from $n$ exists, should be considered as potential successors. More formally, the set of potential successors is defined as: $PSUCC(n)=\{n'=(v', [t'_1, t'_2])\;|\; los(v, v')=true\}$. Please note, that not all successors in $PSUCC(n)$ might be valid w.r.t. time intervals and dynamic obstacles. Indeed, invalid successors should be discarded.

\begin{figure}
    \centering
    \includegraphics[width=\columnwidth]{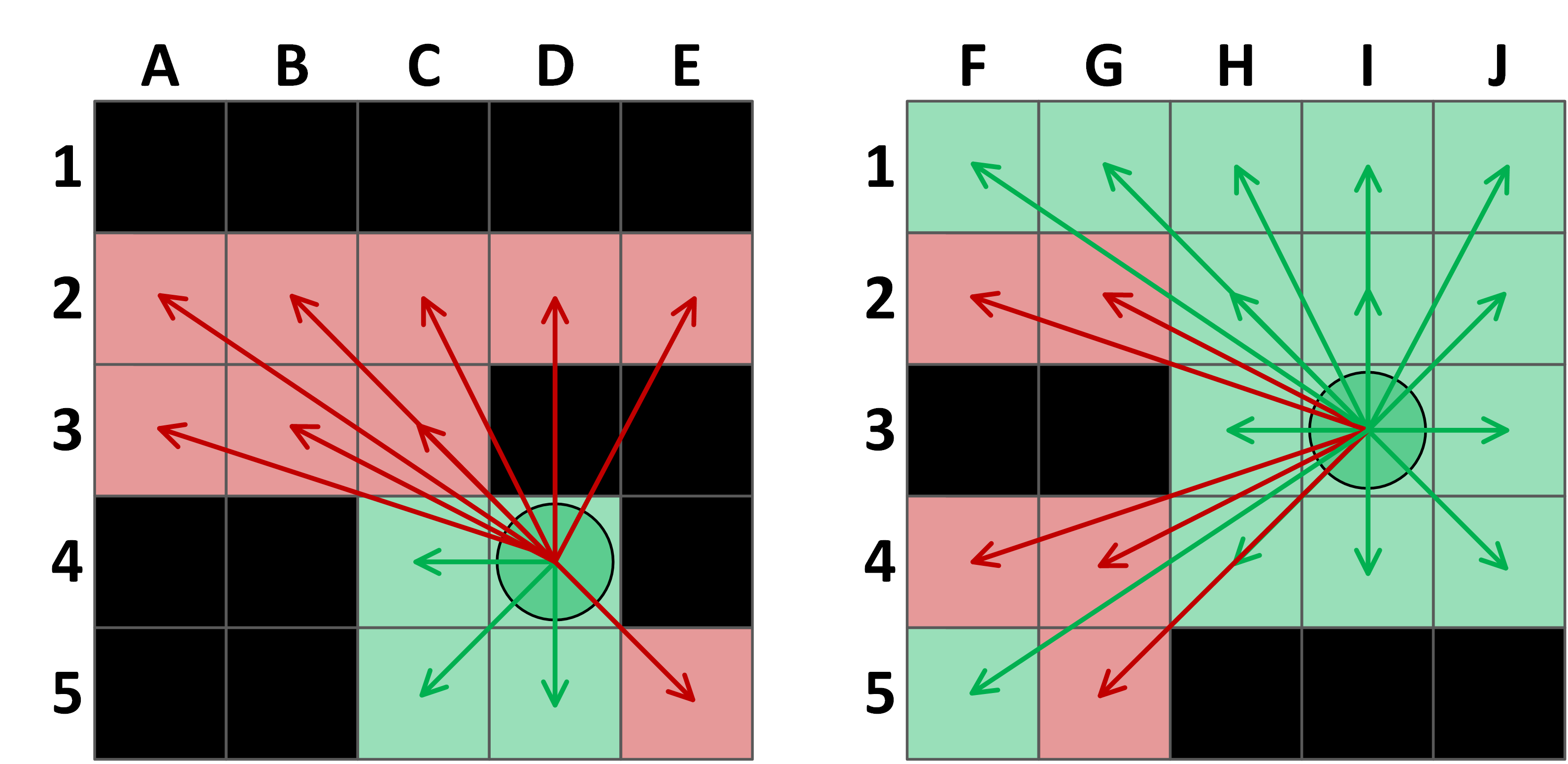}
    \caption{Potential any-angle successors for a node. Cells reachable from $D4$, $I3$ w.r.t. agent's size ($r=0.4$) are shown in green. Arrows symbolize the $los$ function.}
    \label{fig:los_examples}
\end{figure}

We call the algorithm that straightforwardly implements this approach (of generating \emph{all} possible successors for every node under expansion) -- \ntoaasipp. Its pseudocode (which is very similar to the regular A*/SIPP) is shown in Algorithm~\ref{alg:ntoaasipp}. Obviously, considering all possible valid moves between the nodes makes \ntoaasipp complete and optimal. More formally:

\begin{statement}
\ntoaasipp always finds a solution if it exists (and if it does not \ntoaasipp correctly terminates returning \textit{`failure'}) and this solution is optimal, i.e. the resultant plan is time optimal.
\end{statement}

This statement follows from the completeness/optimality of the original SIPP~\cite{phillips2011sipp}.

\section{\itoaasipp}

Generating all possible successors at each expansion as in \ntoaasipp might lead to great computational expenses due to the colossal branching factor as the number of potential successors is proportional to \emph{all} vertices in the graph. Indeed, in some cases, this number is not large as many moves are pruned via the $los$ function, but in general, one can not anticipate this (see Fig.~\ref{fig:los_examples}). Moreover, considering each potential successor is a costly procedure in SIPP that relies on non-trivial reasoning about the time intervals and validity of the move w.r.t. dynamic obstacles. To this end, we suggest ``inverting'' the expansion procedure when not the successors are generated, but rather the predecessors are estimated during the search.

\paragraph{Idea} Assume, that all search nodes are generated beforehand (this is possible due to the use of safe-intervals). The $g$-values of all nodes are initially set to $\infty$. Consider now that after a number of search iterations the $g$-values of some nodes are equal to the true costs of the time-optimal paths from the start. We will refer to such nodes as to the \emph{consistent} ones, bearing the terminology from the LPA* algorithm~\cite{koenig2004lifelong}. Their parents are, obviously, consistent as well. Think now of the remaining inconsistent nodes. The true costs of the time-optimal paths to them are not known, as well as their parents, but it's evident that these true costs might be achieved only by considering transitions from the consistent nodes and that the latter might be thought of as being \textit{potential parents} to them. So, the idea is to consider transitions from the potential parents in a systematic fashion, with the aim of turning inconsistent nodes to consistent ones. When the goal node becomes consistent, the time-optimal solution can be reconstructed by tracing the pointers to the parent nodes.

\SetInd{0.5em}{0.5em}
\begin{algorithm}[t]
\caption{Naive TO-AA-SIPP}
\label{alg:ntoaasipp}
OPEN$:=\{start\}$; CLOSED$:=\oslash$\;
\While{$OPEN \neq \oslash$}
{
    $n:= \argmin_{n \in OPEN} f(n)$\;
    move $n$ from OPEN into CLOSED\;
    \If{$n = goal$}
    {
        \textbf{return} ReconstructPathFromParents($n$)\;
    }
    $PSUCC(n)$ := generate all potential successors of $n$\;
    \For{each node n' in PSUCC(n)}
    {
        \If{$n'$ in CLOSED}
        {
            continue\;
        }
        $g\_new$ = ValidateTransition($n',n$)\;
        \If{$g\_new < g(n')$}
        {
            $g(n') := g\_new$\;
            $f(n') := g\_new + h(n')$\;
            $parent(n') := n$\;
            insert or update $n'$ in OPEN\;
        }
    }
}
        \textbf{return} \textit{path not found}\;
\end{algorithm}

\paragraph{Implementation (high-level)} To implement this idea the following high-level steps should be performed instead of expansion at each iteration of the search:
\begin{enumerate}
  \item Identify the best inconsistent node $n$ and the best potential parent -- $bpp(n)$
  \item Try to decrease $g(n)$ by considering a transition $bpp(n) \rightarrow n$
  \item Estimate whether $n$ became \textit{consistent}. If it did, move it to the set of potential parents
\end{enumerate}

Step 1 chooses the best inconsistent state. It's analogous to choosing the best node in A* (SIPP) for the expansion. As in A* we suggest choosing the node that is likely to belong to the least-cost path from start to goal, i.e. the node with the minimal $f$-value. What is different is that we compute $f$-value as $f(n)=g_{low}(n)+h(n)$, where $g_{low}(n)=g(bpp(n))+h(bpp(n), n)$. By $h(n, n')$ we denote an estimate of time needed to reach $n'$ from $n$. We assume that $h$ satisfies triangle inequality not only w.r.t. the goal node but w.r.t. any node in the search space: $\forall n, n', n''\,:\,h(n, n'') \geq h(n, n') + h(n', n'')$.

Please note, that so far two types of $g$-values are introduced: $g(n)$ which is as an (over)estimate of the cost of the shortest path from start to $n$ (the same as in A*) and $g_{low}(n)$, which is, conceptually, a ``second-order'' estimate of such cost used to compute $f(n)$. The difference is that $g_{low}(n)$ is computed without taking the true cost of the transition from the predecessor into account but rather relying on the consistent estimate of that cost, i.e. $h(bpp(n), n)$, while $g(n)$ is computed based on the complete sequence of feasible transitions from start to $n$. We need a $g_{low}$-value as we do not generate and validate successors as in vanilla A*. Initially, $g_{low}(n)=h(start, n)$ if $los(start,n)=true$ and $g(n)=\infty$ for every $n$ except $start$. For the start node it is: $g_{low}(start)=g(start)=0$.

Step 2 is dedicated to validating the transition to $n$ from its best potential parent, $bpp(n)$. This includes performing all the costly checks that regular SIPP would have done while expanding $bpp(n)$ and generating $n$ as a successor, i.e. checking for the collisions with the dynamic obstacles, computing the earliest arrival time to a destination that avoids collision with the latter etc. This is conceptually similar to performing \emph{lazy} expansions like in Lazy Theta*~\cite{nash2010lazy}. If the transition is valid and $g(bpp(n))+c(bpp(n),n) < g(n)$, where $c$ denotes the cost of a transition, then $g(n)$ is set to the left-hand value (as in A*/SIPP).

Step 3 checks whether after Steps 1 and 2 the node $n$ became consistent, i.e. whether its $g$-value (possibly altered at Step 2) is the cost of the time-optimal plan from the start. To verify this, one needs to establish that there exists no other node in the search space (recall, that initially all the search nodes have already been generated) that being set as the parent of $n$ results in arriving at $n$ earlier. We will explain how to implement this (non-trivial) check further when describing the pseudo-code.

\paragraph{Pseudo-code}
First, all search nodes are generated and their $g$, $g_{low}$, $f$-values are set appropriately -- see Algorithm~\ref{alg:init}. For every node $n$, we choose to explicitly maintain a list of potential parents -- PARENTS($n$). This list contains consistent nodes from which $n$ is reachable w.r.t. static obstacles (we check this via the given $los$-function). Initially, the only node that is added to this list for every $n$ is $start$ (Line 7, Algorithm~\ref{alg:init}). $parent(n)$ is assumed to be the node that was used to compute $g(n)$ and this node is not a part of PARENTS($n$). It is set to \texttt{null} initially for every node except $start$. The latter is deemed consistent and is added to CLOSED. All other nodes of the search space are put in OPEN.

The following steps form the main loop of \itoaasipp -- Algorithm \ref{alg:mail-loop}. The best inconsistent node, $n$, is retrieved from OPEN. If its best potential parent, $bpp(n)$, resides in PARENTS($n$) it is removed from this list. Then a transition from $bpp(n)$ to $n$ is validated (Line 5). At this stage, we perform all the checks the regular SIPP would do when expanding $bpp(n)$ and generating $n$. The outcome is the $g\_{new}$ which is the time by which $n$ is reached via $bpp(n)$. If the transition is invalid $g\_{new} = \infty$. If $g\_{new}$ is less than $g(n)$ then we set $g(n)$ to this value and set $bpp(n)$ to be the parent of $n$ (Lines 6-7). By now we know that there exists a valid plan from $start$ to $n$ via $parent(n)$ and the agent will reach $n$ by time moment $g(n)$. The rest of the code (Lines 8-24) is dedicated to selecting the new best potential parent for $n$ (inside \textit{NewBestPotentialParentExists}) and checking whether $n$ became consistent 
or not
. 

To find the new best potential parent we first set $bpp$ of $n$ to $parent(n)$ (i.e. the node which was used to compute the $g$-value of $n$) at Line 1, Algorithm~\ref{alg:newbppexists}, and then iterate through all the nodes in PARENTS($n$) to identify the one which might decrease $g(n)$. In case such node is absent in PARENTS($n$) we keep $bpp(n)$ equal to $parent(n)$. That explains why we need to perform the check at Line 3 of the main loop. It might be the case that $bpp(n)=parent(n)$ and it is not in PARENTS($n$) so the removal operation at Line 4, Algorithm~\ref{alg:mail-loop} is invalid.

Checking whether $n$ became consistent is split into two parts. First, we check whether there exists a node in CLOSED that might decrease $g(n)$. This is done by invoking \textit{NewBestPotentialParentExists} function at line 8. If such a node exists we re-insert $n$ back to OPEN (Line 9). If no consistent node might be used to decrease $g(n)$, we check the same for the inconsistent nodes (Line 11). We will prove further that the introduced checks allow to verify whether the current node has become consistent or not. If it has, we put it in CLOSED and add it to the set of potential parents for each node $n'$ residing in OPEN s.t. $los(n, n')=true$  -- Lines 15-22.

Finally, a case, when $g(n)$ can not be decreased by a transition from a consistent node, but, at the same time, might be decreased by a transition from an inconsistent node, is handled at Lines 23-24. $n$ is kept in OPEN in that case.

There are two stop criteria for the algorithm. First, if we identify that the goal node became consistent (Line 13) we terminate and return the plan reconstructed using the parent-pointers. Second, if all nodes in OPEN have their $f$-values set to $\infty$ (\texttt{while} condition at Line 1) we terminate inferring that no feasible plan exists (Line 25). 

\paragraph{Similarity to LPA*}
An informed reader is likely to get an impression that the proposed algorithm resembles much the seminal LPA* algorithm~\cite{koenig2004lifelong}. Indeed, \itoaasipp and LPA* share similarities, e.g. LPA* also has two ``g-values'' for every node: the $g$-value itself and the $rhs$-value (right-hand-side-value), corresponding to $g_{low}$ of \itoaasipp. The formula for $g_{low}(n)$/$rhs(n)$ looks the same (however it is not, as $rhs(n)$ relies on the actual cost of the transition from the best predecessor, while $g_{low}(n)$ relies on the heuristic estimate of such cost). Nonetheless, the key idea of \itoaasipp~-- inverting the expansions -- is different from the key idea of LPA* -- re-using the parts of the search-tree to speed up consequent search iterations. Moreover, there are many technical differences between the algorithms. E.g. the $rhs$-value in LPA* can be greater than the $g$-value, while the $g_{low}$-value in \itoaasipp ~-- can not (and this is crucial to the behavior of \itoaasipp as we will show further on). In LPA* a node $n$ is consistent \emph{iff} $g(n)=rhs(n)$, however in \itoaasipp $g(n)=g_{low}(n)$ does not necessarily infer that $n$ is consistent.

\SetInd{0.5em}{0.5em}
\begin{algorithm}[t]
\caption{TO-AA-SIPP Initialization}
\label{alg:init}
OPEN=$\oslash$; CLOSED=$\oslash$\;
Generate all the search nodes\;
\For{each node $n=(v, [t_1, t_2])$}
{
    \If{los(n, start) = true}
    {
    $g_{low}(n):=max(t_1, h(start,n))$\;
    $bpp(n):=start$\;
    PARENTS$(n):=\{start\}$\;
    }
    \Else
    {
    $g_{low}(n):=\infty$\;
    $bpp(n):=null$\;
    PARENTS$(n):=\oslash$\;
    }
    $f(n):=g_{low}(n)+h(n)$\;
    $g(n):=\infty$\;
    $parent(n):=$ null\;
    insert $n$ into OPEN\;
}
$g(start)=0$\;
remove $start$ from OPEN and add it CLOSED\;
\end{algorithm}

\SetInd{0.5em}{0.5em}
\begin{algorithm}[t]
\caption{NewBestPotentialParentExists($n$)}
\label{alg:newbppexists}
$g_{low}(n) := g(n)$; $bpp(n) := parent(n)$\;
$f(n):=g_{low}(n)+h(n)$\;
NewBPPFound:=$false$\;
\For{ \textup{ \textrm{each $n' \in \textrm{PARENTS}(n)$}} }
{
    \If{ $g(n')+h(n, n') < g_{low}(n)$}
    {
        $g_{low}(n)=g(n')+h(n, n')$\;
        $bpp(n)=n'$\;
        $f(n):=g_{low}(n)+h(n)$\;
        NewBPPFound:=$true$\;
    }
}
\textbf{return} NewBPPFound\;
\end{algorithm}

\subsection{Theoretical Properties of \itoaasipp} 

Next we establish that \itoaasipp is complete and optimal.

\begin{lemma}
For any node $n$ it always holds that $g_{low}(n) \leq g(n)$.
\end{lemma}
\label{lemma:g_low-g_n}
\begin{proof}

After the initialization (Algorithm~\ref{alg:init}) this relation holds for all nodes. At each iteration of the main loop (Algorithm~\ref{alg:mail-loop}), the $g$-value of the current node $n$ can be altered only at Line 7. It is set to $g\_new$ which equals $g(bpp(n))+c$, where $c$ is the cost of the collision-free transition from $bpp(n)$ to $n$ (computed at Line 5). Note that $c \geq h(bpp(n), n)$ due to admissible and consistent heuristic. At the same time $g_{low}(n)=g(bpp(n)) + h(bpp(n), n)$, thus $g_{low} (n) \leq g(n)$ after updating $g(n)$ at Line 7. Next, $g_{low}(n)$ might be altered at Line 8 inside the \textit{NewBestPotentialParentExists} function. When it is called $g_{low}(n)$ is first set to $g(n)$ (Line 1, Algorithm~\ref{alg:newbppexists}) and then might only decrease (Line 5, Algorithm~\ref{alg:newbppexists}), while $g(n)$ stays unaltered. Thus after Line 8 of the main loop $g_{low}(n) \leq g(n)$. In no subsequent line of code, the $g$/$g_{low}$-value of the currently explored node $n$ is updated. Thus for node $n$ the statement of the Lemma holds.

One may also note that at Line 19 of the main loop the $g_{low}$-value of some other node, $n'$ is updated. Observe, however, that its $g_{low}$-value might only decrease, while its $g$-value stays unaltered. This infers that for that node the statement of the Lemma also holds.

\end{proof}

\SetInd{0.5em}{0.5em}
\begin{algorithm}[t]
\caption{\itoaasipp Main Loop}
\label{alg:mail-loop}
\While{ \textrm{ $\min_{n \in \textrm{OPEN}} f(n) < \infty$ } }
{
    $n:= \argmin_{n \in \textrm{OPEN}} f(n)$; remove $n$ from OPEN\;
    \If{ \textup{$bpp(n) \in \textrm{PARENTS}(n)$} }
    {
        remove $bpp(n)$ from PARENTS$(n)$\;
    }
    $g\_{new}:=$ ValidateTransition$(n, bpp(n))$\;
    \If{$g\_{new}<g(n)$}
    {
        $g(n) := g\_{new}$; $parent(n):=bpp(n)$\;
    }

    \If{NewBestPotentialParentExists($n$)}
    {
        insert $n$ into OPEN\;
        continue\;
    }
    \If{g(n)+h(n) $\leq$ min f-value in \textup{OPEN}}
    {
        insert $n$ into CLOSED\;
        \If{$n = goal$}
        {
            \textbf{return} ReconstructPathFromParents($n$)\; 
        }
        \For{each $n' \in$ \textup{OPEN}}
        {
            \If{$los(n, n') = true$}
            {
                insert $n$ into PARENTS$(n')$\;
                \If{$g(n)+h(n,n') < g_{low}(n')$}
                {
                    $g_{low}(n') := g(n) + h(n,n')$\;
                    $bpp(n') := n$\;
                    $f(n') := g_{low}(n') + h(n')$\;
                    update $n'$ in OPEN\;
                }
            }
            
        } 
    } 
    \Else
    {
        insert $n$ into OPEN\;
    }
}
\textbf{return} \textit{path not found}\;
\end{algorithm}

\begin{lemma}
For any $i$: $g_{low}^i(n) \leq g_{low}(n)$, where $g_{low}^i(n)$ is the $g_{low}$-value of the node $n$ extracted from OPEN at the beginning of the $i$-th iteration of the \itoaasipp main loop and $g_{low}(n)$ is the $g_{low}$-value of $n$ at the end of the iteration.
\end{lemma}
\label{lemma:g_low_non_decreasing}
\begin{proof}

Consider a node $n$ which is extracted from OPEN with the $g_{low}$-value equal to $g_{low}^i(n)$. The only place where the $g_{low}$-value is altered is inside the \textit{NewBestPotentialParentExists} function. It first is set to be equal to $g(n)$ (Line 1, Algorithm~\ref{alg:newbppexists}) which is greater or equal to $g_{low}^i(n)$ (Lemma 1). Thus, if the condition at Line 5 of Algorithm~\ref{alg:newbppexists} is never met further on (meaning that the $g_{low}$-value does not change anymore), the claim holds. 

Suppose that the condition at Line 5 is met, and $g_{low}$ is decreased at Line 6 to some value that we denote here as $g_{low}'(n)$. Moreover, suppose that $g_{low}'(n) < g_{low}^i(n)$. This means that at the current iteration of the main loop there exists a node in PARENTS($n$), $n'$, which delivers a better $f$-value to $n$. But the only place in the code where a node is added to PARENTS is Line 17 of Algorithm~\ref{alg:mail-loop} and in case the new parent is better than the existing ones (i.e. it can be used to achieve minimal $g_{low}$) it is set to be $bpp(n)$ and the $f$-value is updated accordingly (Lines 18-21, Algorithm~\ref{alg:mail-loop}). Thus $n$ with that $f$-value, equal to $g_{low}'(n) + h(n)$, which is lower than $g_{low}^i(n) + h(n)$, should have been extracted from OPEN either at the beginning of the current iteration (which leads to contradiction) or at some previous iteration. In the latter case $n'$ should have been removed from PARENTS($n$) before iteration $i$. This contradiction concludes the proof.

\end{proof}

\begin{lemma}
The sequence $\{ f_{min}^1, f_{min}^2, ..., f_{min}^K\}$, where $f_{min}^i$ is the $f$-value of the node extracted from OPEN at the $i$-th iteration of the search, is non-decreasing.
\end{lemma}

\begin{proof}
First, recall that all nodes of the search space are generated during the initialization (Algorithm~\ref{alg:init}) so no new nodes can be encountered during the search. Next, observe that at any iteration of the main loop, denote it $i$, a single node is extracted from OPEN -- the one with the minimum $f$-value equal to $f_{min}^i$, and can be put either back to OPEN (lines 9  and 24 of the Algorithm~\ref{alg:mail-loop}) or to CLOSED (with the possibly updated $f$-value). Let's consider first the cases when $n$ is put back to OPEN.

\textit{Case 1 (Lines 8-10)}. In this case, \textit{NewBestPotentialParent(n)} returned true. This means that the $g_{low}$-value of $n$ was altered. However, due to (Lemma 2) this updated $g_{low}$-value can not be less than $g_{low}^i$, i.e. the $g_{low}$-value that was used for computing $f_{min}^i$.
This infers that $n$ was put back to OPEN with the larger (or equal) $f$-value and now OPEN contains nodes whose $f$-values are greater or equal $f_{min}^i$.

\textit{Case 2 (Line 24)}. In this case, \textit{NewBestPotentialParent(n)} returned false which means that there exists no node in CLOSED that can be used as a parent for $n$ to potentially decrease $g(n)$. Moreover, the check at line 11 returned false, which means that $g(n)+h(n) > f_{min} \geq f_{min}^i$, where $f_{min}$ is the lowest $f$-value in OPEN currently. As Algorithm~\ref{alg:newbppexists} has set the $g_{low}$-value of $n$ to be $g(n)$ before this check we can infer that now (at line 24) $g_{low}(n)+h(n) > f_{min}$, thus $n$ is inserted back to OPEN with the $f$-value strictly exceeding $f_{min}$. Thus at the next iteration, a node with an $f$-value that is greater or equal to $f_{min}^i$ will be extracted.

Finally, let's consider the case when $n$ was retrieved from OPEN and put to CLOSED (lines 11-22). In that case the $f$-values of some nodes residing in OPEN might be updated (line 22). Let $n'$ be such a node. Its $f$-value now equals (see lines 19 and 21) $f(n')=g(n) + h(n, n') + h(n')$. Due to the consistency of $h$ we infer that $f(n') \geq g(n) + h(n)$. Moreover, as Lemma~1 prescribes $g(n) + h(n) \geq g_{low}(n)+h(n) = f_{min}^i$. Overall $f(n') \geq f_{min}^i$. Thus no updated node in OPEN has an $f$-value that is lower than $f_{min}^i$.

This concludes the proof, as we have shown that no matter whether we put the node $n$ back to OPEN or to CLOSED, at the end of the iteration the minimum $f$-value in OPEN is greater or equal to $f_{min}^i$.

\end{proof}

\begin{lemma}
CLOSED contains nodes for which the time-optimal path from the start node is known.
\end{lemma}
\begin{proof}

We prove by induction that before each iteration of the main loop CLOSED contains nodes for which the time-optimal path from the start is known.

\textit{k=1. Base case.} Before the first iteration CLOSED contains only one node, $start$, for which the statement obviously holds.

\textit{k=m. Induction hypothesis.} Assume that before the $m$-th iteration of the main loop the statement holds.

\textit{k=m+1. Induction step.} We prove now that after the $m$-th iteration of the main loop (i.e. before the $(m+1)$-th iteration) for all nodes in CLOSED the time-optimal paths from the $start$ are known.

Observe that only a single node might be added to CLOSED at a single iteration of the main loop -- the one extracted from OPEN. If this node was not added at the $m$-th iteration, the claim holds due to the induction hypothesis. Assume now that the node $n$ was extracted from OPEN and added to CLOSED at the $m$-th iteration. We need to show that the time-optimal path from $start$ to $n$ is known, which is equivalent to showing that no other node in the entire search space can be used as a parent of $n$ to decrease the cost of the path from $start$ to $n$, which is $g(n)$.

Note first that if $n$ was inserted to CLOSED then \textit{NewBestPotentialParentExisits(n)} function returned false (otherwise the algorithm won't reach Line 12 due to \texttt{continue} at Line 10). In order to return false this function has to iterate through all nodes $n'$ in PARENTS($n$) and for each such node verify that $g(n') + h(n, n') \geq g_{low}(n)=g(n)$. Recall now that \textit{i}) the $g$-values of all nodes residing in PARENTS($n$), can not be decreased due to the induction hypothesis, and \textit{ii}) $h$ is a consistent heuristic. This means that no node from PARENTS($n$), which is a subset of CLOSED, can be used to decrease $g(n)$. Transitioning to $n$ from any other node $n'$ in CLOSED, which is not part of PARENTS($n$), can not provide a better value for $g(n)$ than it is at the considered iteration, because the transitions $n' \rightarrow n$ have already been considered at previous iterations (due to we always add a node that became consistent to the list of potential parents of all other reachable nodes in OPEN).

We will show now that no node from the OPEN part of the state-space can be used to decrease $g(n)$ as well. Observe that if $n$ was inserted to CLOSED then the check at Line 11 was true, which means that $g(n) + h(n) \leq g_{low}(n_{best}) + h(n_{best})$, where $n_{best}$ is the node with the lowest $f$-value in OPEN. For any other node $n'$ residing in OPEN it, obviously, holds that $g_{low}(n_{best})+h(n_{best}) \leq g_{low}(n') + h(n')$. Due to the consistency of $h$ we have: $g_{low}(n') + h(n') \leq g_{low}(n') + h(n,n') + h(n)$. Thus, for any node $n'$ in OPEN it holds that: $g(n) + h(n) \leq  g_{low}(n') + h(n,n') + h(n)$ which is equivalent to $g(n) \leq g_{low}(n') + h(n,n')$. The latter guarantees that no node from OPEN can be used to decrease $g(n)$ at the current iteration of the main loop. Moreover, as the $f$-value of the best element in OPEN is not decreasing from iteration to iteration (Lemma 3), we infer that $g(n)$ can not be lowered down anymore at any future iteration as well. This concludes the proof.
\end{proof}

\begin{lemma}
When a node $n$, s.t. $bpp(n)=parent(n)$, is extracted from OPEN then it will be added to CLOSED at the current iteration.
\end{lemma}

\begin{proof}

Let $n$ be such a node. It holds at the beginning of the iteration that $g(n)=g_{low}(n)$ (as  $bpp(n)=parent(n)$). $g\_new$ computed at Line 5 will be equal to $g(n)$ thus the latter will be unaltered (and will stay equal to $g_{low}$). \emph{NewBestPotentialParentExists} will return false as otherwise $g_{low}(n)$ would have been decreased which contradicts Lemma 2. Consequently, we proceed to the check at Line 11 which will be passed as $g(n)$ was not altered at this iteration. Thus, $n$ will be added to CLOSED.

\end{proof}

\begin{lemma}
Every reachable node in the search-space will be eventually added to CLOSED in case the goal detection block is omitted (i.e. Lines 13-14 are removed from the Algorithm~\ref{alg:mail-loop}) and after the finite number of iterations the algorithm will terminate.
\end{lemma}
\begin{proof}

We first show that the algorithm terminates after the finite number of iterations. Recall that all nodes in the search space are generated during initialization. Each node may be added to CLOSED only once. Hence, it may be added to the PARENT-list of any other node only once. Thus, the number of potential parents for node $n$ is limited by $N-1$, where $N$ is the total number of nodes. Every time $n$ extracted from OPEN one of the nodes from PARENTS($n$) is removed, except the case when $n$ is extracted with $bpp(n)=parent(n)$. However, in this case, such iteration leads to removing $n$ from OPEN and adding it to CLOSED (Lemma 5). Thus, the number of the iterations of the main loop at which $n$ was extracted with non-empty PARENTS$(n)$ is finite. Therefore, the maximum number of extractions of a single node from OPEN is finite. 

We now show that after the finite number of extractions of $n$ from OPEN it is either kept in OPEN with $f(n)=\infty$ or is put in CLOSED. To show that we consider separately two possibilities.

\textit{Case 1}. $g(n)$ is not set to a finite value every time $n$ is extracted from OPEN (meaning that the transition from $bpp(n)$ is not valid). Think of the iteration when PARENTS($n$) becomes empty for good (this will, indeed, happen as shown above). At this iteration $g_{low}(n)$ is set to $g(n)$ which is infinite (at Line 1, Algorithm~\ref{alg:newbppexists}). It infers that $n$ is put to OPEN with $f(n)=\infty$.

\textit{Case 2}. $g(n)$ becomes finite at some iteration. It means that $parent(n)$ has been identified. It guarantees that $bpp(n)$ may not become \texttt{null} at any further iteration. Thus, even if PARENTS($n$) becomes empty (or contains only the nodes that can not decrease the $g$-value anymore) $n$ will be extracted from OPEN with $bpp(n)=parent(n)$ and added to CLOSED (Lemma 5).

Thus, every node $n$ after the finite number of iterations will be either in CLOSED or in OPEN with $f(n)=\infty$. In both considered cases the algorithm will terminate due to the \texttt{while} condition at Line 1 of the main loop (assuming that $min$ of the empty set returns $\infty$).

Assume now that after the termination of the algorithm some reachable node $n$ was not added to CLOSED, meaning it still resides in OPEN. In case its $f$-value is finite this is impossible as it contradicts with the \texttt{while} condition at Line 1, Algorithm~\ref{alg:mail-loop}. Assume now that $f(n)$ is infinite. The only reason $n$ may reside in OPEN, in that case, is that no suitable parent was identified that delivers a finite $g$-value for $n$. 
However, we know that at least one such node exists (the one that proceeds $n$ in a valid path from $start$ which we assume is existent). Denote such node $n'$. The only reason $n'$ has not become a parent for $n$ maybe that $n'$ was not added to CLOSED (as otherwise $n'$ would have been added to PARENTS($n$) and the transition $n' \rightarrow n$ would have been considered at some iteration). Recursively applying the same reasoning to $n'$ we will find that $start$ is not added to CLOSED which leads to the contradiction as $start$ is added to CLOSED during the initialization (Line 17, Algorithm~\ref{alg:init}).

\end{proof}

\begin{theorem}
iTO-AA-SIPP is complete and optimal.
\end{theorem}
\begin{proof}

\textit{Completeness}. To prove completeness we need to show that: \textit{1}) if a solution exists \itoaasipp will find it; \textit{2}) if there is no solution it will correctly terminate returning `path not found'.

\textit{Case 1}. If a solution exists then $goal$ is reachable and will be eventually added to CLOSED in accordance with Lemma 6. When this occurs, \itoaasipp will stop due to the check at Line 13, Algorithm~\ref{alg:mail-loop}. A valid path to $goal$ obtained from tracing back the parent-pointers will be returned (Line 14).

\textit{Case 2}. If no solution exists then $goal$ is not reachable and the stop criterion at Line 13 will never be met. Thus the behavior of \itoaasipp is analogous to the one established by Lemma 6. I.e. the main loop will be executed a finite number of times and the algorithm will terminate returning `path not found' (Line 25).

\textit{Optimality} is a direct corollary of Lemma 4.

\end{proof}

\begin{figure*}[t]
    \centering
    \includegraphics[width=\linewidth]{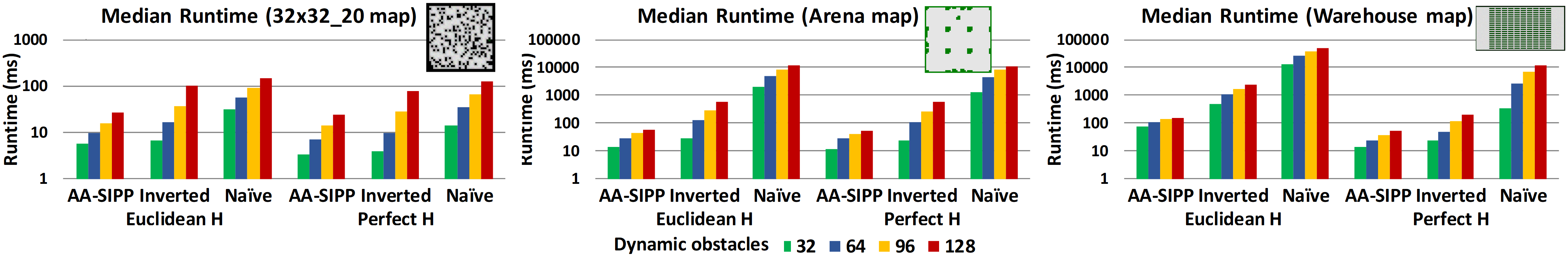}
    \caption{Median runtime of the algorithms (the scale is logarithmic).}
    \label{figExpRuntime}
\end{figure*}

\begin{table*}[t]
    \centering
    \fontsize{9}{10}\selectfont
    \renewcommand{\arraystretch}{1.2}
    \begin{tabular}{c|ccccc|ccccc}
        &\multicolumn{5}{c|}{Euclid H}&\multicolumn{5}{c}{Perfect H}\\
        \hline
         \makecell[t]{\# of\\obs.}&  \makecell[t]{AA-SIPP\\Iterations} & \makecell[t]{AA-SIPP\\VT-calls} & \makecell[t]{Inverted\\Iterations} & \makecell[t]{Na\"ive\\Iterations} & \makecell[t]{Na\"ive\\VT-calls}
         &  \makecell[t]{AA-SIPP\\Iterations} & \makecell[t]{AA-SIPP\\VT-calls} & \makecell[t]{Inverted\\Iterations} & \makecell[t]{Na\"ive\\Iterations} & \makecell[t]{Na\"ive\\VT-calls}\\
         \hline
         32 & 885 & 7 357 & 7 348 & 856 & 391 845 & 183 & 1 764 & 1 377 & 18 & 13 019\\
         64 & 883 & 8 704 & 11 953 & 871 & 461 357 & 218 & 2 345 & 1 998 & 84 & 53 488\\
         96 & 945 & 9 626 & 16 884 & 908 & 484 819 & 275 & 2 977 & 3 015 & 154 & 92 125\\
         128 & 991 & 10 299 & 24 088 &931 & 535 076 & 325 & 3 616 & 4 283 & 222 & 129 981\\
    \end{tabular}
    \caption{Median number of iterations and \textit{ValidateTransition} calls on the \texttt{warehouse} map.}
    \label{tab:Steps}
\end{table*}

\section{Empirical Evaluation}
\paragraph{Setup} We implemented \ntoaasipp and \itoaasipp and evaluated them on a range of different grid-maps with a varying number of dynamic obstacles\footnote{Our implementation and the experimental data are available at \texttt{github.com/PathPlanning/TO-AA-SIPP.}}. Non-optimal any-angle SIPP planner, AA-SIPP, which is a part of the prioritized multi-agent solver AA-SIPP(m)~\cite{yakovlev2017aasipp}, was also evaluated. The agent and the dynamic obstacles were represented as disks of radius 0.5 (of a grid cell) and their movement speed was 1.0.

Three different maps from MovingAI benchmark~\cite{sturtevant2012, stern2019} were used: \texttt{Arena} -- a $49 \times 49$ map composed of a small number of static obstacles and large open areas; \texttt{32x32\_20} -- a $32 \times 32$ map with 20\% of randomly blocked cells; \texttt{Warehouse} -- a $170 \times 84$ map from a logistics domain. We chose these maps as they represent different types of environments and differ in size.

500 scenarios were generated for each map. Each scenario contained the randomly chosen start and goal locations for the agent as well as 128 trajectories of the dynamic obstacles, that are collision-free and contain any-angle moves. These trajectories were obtained by invoking the prioritized multi-agent planner from~\cite{yakovlev2017aasipp}. During the evaluation, we used the scenarios as follows. We took the start-goal location pair for the agent from a scenario as well as the trajectories of the first 32-62-96-128 dynamic obstacles. Next, we ran all algorithms on such an instance and proceeded to the next scenario.

We ran the algorithms with two heuristics: Euclidean distance and a pre-computed perfect heuristic that takes static obstacles into account (denoted as \texttt{Perfect H}). To compute it we ran Dijkstra's algorithm backward from the goal node and at each iteration generated all possible any-angle successors (like in \ntoaasipp). 

\paragraph{Results} Median runtimes of the algorithms are presented in \figurename~\ref{figExpRuntime} (the Y-axis of the plots is logarithmic). As expected, the runtime of \ntoaasipp is the worst in all cases. The difference between \ntoaasipp and \itoaasipp is less pronounced for \texttt{32x32\_20} map but for the two other maps \itoaasipp is faster by one order of magnitude. Noteworthy that in certain setups, e.g. the \texttt{32x32\_20} map with the perfect heuristic and the small number of dynamic obstacles, the runtime of \itoaasipp is comparable of that of AA-SIPP, however, in general, the latter is evidently faster.

The impact of using the perfect heuristic depends largely on the topology of the map. A significant boost in performance (for all algorithms) is observed on the map with randomly blocked cells and the \texttt{Warehouse} map. On the \texttt{Arena} map Euclidean distance is a rather accurate and informative heuristic thus the advantage of \texttt{Perfect H} is less pronounced. 

Table~\ref{tab:Steps} shows the median number of iterations and \textit{ValidateTransition} calls for the \texttt{Warehouse} map (for the other maps we observed similar trends so we present a single table here for the sake of brevity). For \itoaasipp the number of iterations equals the number of VT-calls so no dedicated column for this is present. As one can note, the number of iterations for \ntoaasipp is quite similar to the one of AA-SIPP and is much less compared to \itoaasipp, however, the number of VT-Calls, which are quite computationally expensive, for \ntoaasipp exceeds hundreds of thousands. \itoaasipp uses one order of magnitude fewer VT-calls to find optimal solutions. The number of VT-calls for AA-SIPP is only slightly lower in most cases. However, recall that \itoaasipp performs, aside from validating the transition, such operations as finding the new best potential parent, adding a new parent for the nodes in OPEN when we add a node to CLOSED, etc. That is why the overall runtime of \itoaasipp is higher.

One can also note that the number of iterations for \ntoaasipp and AA-SIPP algorithms using Euclidean heuristic grows slowly with the number of dynamic obstacles. We believe that the main reason for this is that on the \texttt{warehouse} map Euclidean distance provides largely inaccurate cost-to-go estimates in most cases due to the prolonged static obstacles. Thus the algorithm has to expand lots of states no matter how many dynamic obstacles are present in the environment.

Indeed, both \ntoaasipp and \itoaasipp always found solutions of the same cost (being time-optimal solutions) which was less than or equal to the cost of the AA-SIPP solutions. The averaged difference between the costs is shown in \figurename~\ref{figExpCost}. It is known that in static environments the difference in cost between the solutions found by greedy any-angle algorithms (e.g. Theta*) and optimal ones (e.g. ANYA) is less than a percent on average \cite{harabor2016}. Here we observe similar figures. 
It is also noteworthy, that for some instances the cost-difference is much more pronounced. For example, the maximum differences on the \texttt{32x32\_20}, \texttt{Arena} and \texttt{Warehouse} maps were 73.9\%, 16.3\% and 5.2\% respectively. An example of two solutions of largely varying costs is shown in \figurename~\ref{fig:aasipp_vs_toaasipp} (the animation is available at \texttt{https://youtu.be/k245e3CMUO4}). AA-SIPP was not able to discover the transition framed in yellow, as to find it the algorithm had to first move from the source cell to one of the adjacent cells (shown with red crosses), but these moves were invalid due to the dynamic obstacles (not shown in the figure). \itoaasipp indeed discovered this move and used it to build an optimal plan.

\begin{figure}[t]
    \centering
    \includegraphics[width=0.8\columnwidth]{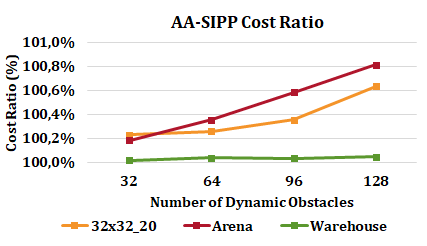}
    \caption{The ratio between optimal costs and costs of trajectories found by greedy AA-SIPP algorithm.}
    \label{figExpCost}
\end{figure}

\section{Related Works}
Two lines of research are the most relevant to this work. The first one deals with any-angle path finding in static environments and the second one -- with graph-based path planning in environments with dynamic obstacles. 

The most widely known algorithm for (sub-optimal) any-angle path finding in a static environment (represented as a grid) is apparently Theta*~\cite{nash2007}. This algorithm was modified and enhanced in numerous works~\cite{nash2010lazy, oh2016strict}, etc. Other prominent algorithms of that kind are Block A*\cite{yap2011any} and Field D*~\cite{ferguson2006using}. Closely related are the ones that rely on $2^k$-connectivity of the grid~\cite{rivera2020}. 

The first provably optimal any-angle path planner for static grids is ANYA \cite{harabor2016}. Another optimal planner was introduced in~\cite{vsivslak2009accelerated} but only a conjecture was made regarding its optimality. It is noteworthy that \itoaasipp shares some similarities with the planner introduced in that paper as the latter also reasons about the potential parents represented as the subset of the grid cells encountered during the search.

\begin{figure}[t]
    \centering
    \includegraphics[width=0.75\columnwidth]{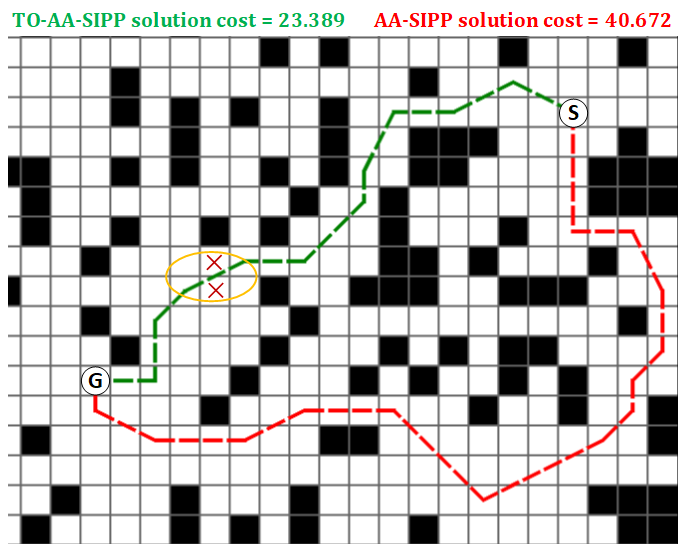}
    \caption{An instance for which the cost of AA-SIPP solution equals 173\% of the optimal cost (achieved by \itoaasipp). Dynamic obstacles are not shown.}
    \label{fig:aasipp_vs_toaasipp}
\end{figure}

The other line of research, related to this work, is graph-based path planning in environments with dynamic obstacles. In~\cite{likhachev2009planning} it was suggested to treat the dynamic obstacles as static and re-plan at high rate. In~\cite{silver2005} it was suggested to run A* with time-reservation table when one agent needs to avoid the trajectories of the other agents (which can be viewed as dynamic obstacles). The idea of grouping distinct time steps into the intervals -- safe interval path planning (SIPP) -- was proposed in~\cite{phillips2011sipp}. SIPP was enhanced to handle any-angle moves in~\cite{yakovlev2017aasipp} but the resulting algorithm is not optimal w.r.t. such moves. The algorithms suggested in this work indeed belong to the SIPP family and provide provably optimal solutions for any-angle path planning with dynamic obstacles.

Finally, one might trace the similarities between \itoaasipp and LPA*~\cite{koenig2004lifelong}. Nonetheless, these algorithms are quite different as we explained earlier in the paper. R* algorithm~\cite{likhachev08rstar} might also be noted in that context as it also reasons over the set of potential parents for a search node and uses the term $g_{low}(n)$ in a similar sense as we do. Still, the idea behind R* (randomizing the search to avoid local minima) is fundamentally different from that of \itoaasipp.

\section{Summary}
In this paper, we have considered the problem of finding time-optimal any-angle paths in the presence of moving obstacles and proposed two algorithms that can obtain provably optimal solutions. A prominent direction of future research is developing more efficient optimal solvers as we believe that the considered problem demands more elegant solutions.

The presented algorithms might also be of a certain value to the multi-agent path finding (MAPF). E.g. the suggested \itoaasipp planner can be used as the low-level solver for the MAPF algorithms of the conflict-based search (CBS) family that support non-uniform cost moves into arbitrary (any-angle) directions: ECBS-CT~\cite{cohen2019optimal}, CCBS~\cite{andreychuk2019multi}, CBS+TAB~\cite{walker2020bipartite}, etc. Note, however, that in these algorithms the low-level planner is invoked thousands of times so its efficiency should be high. We believe that \itoaasipp in its current form is not fast enough to be straightforwardly used in CBS-solvers. This, again, motivates the development of more advanced algorithms that can find time-optimal any-angle paths in the presence of dynamic obstacles.

\section{Acknowledgments} 

In the first place, we would like to thank the reviewers for investing their time in thoroughly reading the preliminary version of the manuscript and providing detailed comments and suggestions, which helped improve this paper considerably. This work was supported by RFBR grant \#20-57-00011. Anton Andreychuk is also supported by the RUDN University Strategic Academic Leadership Program.

\bibliography{aaai21}

\end{document}